\documentclass[10pt,twocolumn,letterpaper]{article}

\usepackage{cvpr}
\usepackage{times}
\usepackage{epsfig}
\usepackage{graphicx}
\usepackage{amsmath}
\usepackage{amssymb}

\usepackage{amsthm}
\theoremstyle{definition}
\newtheorem{dfn}{Definition}[section]

\theoremstyle{plain}

\newtheorem{lem}{Lemma}[section]

\usepackage{framed}	
\usepackage{array}
\usepackage{caption}
\usepackage{subcaption}
\usepackage{epstopdf}
\DeclareGraphicsExtensions{.eps}
\usepackage{url}

\usepackage{multirow}
\usepackage{booktabs}

\usepackage{algpseudocode}
\usepackage{algorithm}

 \cvprfinalcopy 

\def\cvprPaperID{900} 

\ifcvprfinal\fi
\begin{document}

\title{Iteratively Reweighted Graph Cut for Multi-label MRFs with Non-convex Priors}

\author{Thalaiyasingam Ajanthan, Richard Hartley, Mathieu Salzmann, and Hongdong Li\\[0.5em]
Australian National University and NICTA\thanks{NICTA is funded by the Australian Government as represented by the Department of Broadband, Communications and the Digital Economy and the Australian Research Council through the ICT Centre of Excellence program.}\\
Canberra, Australia
}

\maketitle

\begin{abstract}
While widely acknowledged as highly effective in computer vision, multi-label MRFs with non-convex priors are difficult to optimize. To tackle this, we introduce an algorithm that iteratively approximates the original energy with an appropriately weighted surrogate energy that is easier to minimize. Our algorithm guarantees that the original energy decreases at each iteration. In particular, we consider the scenario where the global minimizer of the weighted surrogate energy can be obtained by a multi-label graph cut algorithm, and show that our algorithm then lets us handle of large variety of non-convex priors. We demonstrate the benefits of our method over state-of-the-art MRF energy minimization techniques on stereo and inpainting problems.
\end{abstract}

\section{Introduction}
In this paper, we introduce an algorithm to minimize the energy of multi-label Markov Random Fields (MRFs) with non-convex edge priors. In general, minimizing a multi-label MRF energy function is NP-hard. While in rare cases a globally optimal solution can be obtained in polynomial time, \eg, in the presence of convex priors~\cite{ishikawa2003exact}, in most scenarios one has to rely on an approximate algorithm. Even though graph-cut-based algorithms~\cite{boykov2001fast} have proven successful for specific problems (\eg \textit{metric} priors), there does not seem to be a single algorithm that performs well with different non-convex priors such as the truncated quadratic, the Cauchy function and the corrupted Gaussian.

Here, we propose to fill this gap and introduce an iterative graph-cut-based algorithm to minimize multi-label MRF energies with a certain class of non-convex priors. Our algorithm iteratively minimizes a weighted surrogate energy function that is easier to optimize, with weights computed from the solution at the previous iteration. We show that, under suitable conditions on the non-convex priors, and as long as the weighted surrogate energy can be decreased, our approach guarantees that the true energy decreases at each iteration. 

More specifically, we consider MRF energies with arbitrary data terms and where the non-convex priors are concave functions of some convex priors over pairs of nodes. In this scenario, and when the label set is linearly ordered, the solution at each iteration of our algorithm can be obtained by applying the multi-label graph cut algorithm of~\cite{ishikawa2003exact}. Since the resulting solution is optimal, our algorithm guarantees that our MRF energy decreases. 

In fact, our method is inspired by the Iteratively Reweighted Least Squares (IRLS) algorithm which is well-known for continuous optimization. To the best of our knowledge, this is the first time that such a technique is transposed to the MRF optimization scenario.

We demonstrate the effectiveness of our algorithm on the problems of stereo correspondence estimation and image inpainting. Our experimental evaluation shows that our method consistently outperforms other state-of-the-art graph-cut-based algorithms~\cite{boykov2001fast, veksler2012multi}, and, in most scenarios, yields lower energy values than TRW-S \cite{kolmogorov2006convergent}, which was shown to be one of the best-performing multi-label approximate energy minimization methods~\cite{szeliski2008comparative,kappes2013comparative}.

\subsection{Related work}
Over the years, two different types of approximate MRF energy minimization methods have been proposed. The first class of such methods consists of move-making techniques that were inspired by the success of the graph cut algorithm at solving binary problems in computer vision. These techniques include $\alpha$-expansion, $\alpha$-$\beta$ swap~\cite{boykov2001fast} and multi-label moves~\cite{veksler2012multi, torr2009improved, jezierska2011fast}. The core idea of these methods is to reduce the original multi-label problem to a sequence of binary graph cut problems. Each graph cut problem can then be solved either optimally by the max-flow algorithm~\cite{boykov2004experimental} if the resulting binary energy is submodular, or approximately via a roof dual technique~\cite{boros2002pseudo} otherwise. The second type of approximate energy minimization methods consists of message passing algorithms, such as belief propagation (BP)~\cite{felzenszwalb2006efficient}, tree-reweighted message passing (TRW)~\cite{wainwright2005map,kolmogorov2006convergent} and the dual decomposition-based approach of~\cite{komodakis2011mrf}, which TRW is a special case of. 

As mentioned above, our algorithm is inspired by the IRLS method. Recently, several methods similarly motivated by the IRLS have been proposed to minimize different objective functions. For instance, in~\cite{aftabgeneralized}, the $L_q$ norm (for $1 \le q < 2$) was minimized by iteratively minimizing a weighted $L_2$ cost function. In~\cite{ochs2013iterated}, an iterated $L_1$ algorithm was introduced to optimize non-convex functions that are the sum of convex data terms and concave smoothness terms. More recently, a general formulation (not restricted to weighted $L_2$ or $L_1$ minimization) was studied together with the conditions under which such iteratively reweighted algorithms ensure the cost to decrease~\cite{aftabwacv15}. In the next section, we propose an extension of this formulation that will later allow us to tackle the case of multi-label MRFs.

\section{Iteratively reweighted minimization}\label{sec:algo}

Given a set $\mathcal{X}$ and functions $f_{i}: \mathcal{X} \rightarrow \mathcal{D}$ and $h_i: \mathcal{D} \rightarrow {\rm I\!R}$, where $\mathcal{D} \subseteq {\rm I\!R}$, let us assume that we want to minimize an objective function of the form
\begin{equation}
C_h(\mathbf{x}) = \sum_{i=1}^k h_i \circ f_{i}(\mathbf{x})\ ,
\label{eqn:hcost}
\end{equation}
and, without loss of generality, that we have a method to minimize a weighted cost function of the form
\begin{equation}
C_w(\mathbf{x}) = \sum_{i=1}^k w_{i}\, f_{i}(\mathbf{x})\ .
\label{eqn:wcost}
\end{equation}
For instance, in the IRLS algorithm, $f_i(\mathbf{x})$ is a squared cost.

Our goal is to study the conditions under which $C_h$ can be minimized by iteratively minimizing $C_w$.
To this end, we first give the definition of a supergradient, which we will rely upon in the following discussion.

\begin{dfn} 
Let $\mathcal{D}$ be a subset of ${\rm I\!R}$. A supergradient of a function $h: \mathcal{D} \rightarrow {\rm I\!R}$ at a point $c$ is a value $h^s(c) \in {\rm I\!R}$ such that $h(d) \le h(c) + (d-c)\, h^s(c)$ for any point $d$.
\end{dfn}
A supergradient $h^s$ is called a strict supergradient if the inequality is strict for any point $d \ne c$. If the function is differentiable, then the supergradient at a point is unique and equal to the derivative. A concave function defined on a subset of the real numbers has a supergradient at each interior point.

In~\cite{aftabwacv15}, the following lemma was provided to study the behavior of iteratively reweighted minimization methods.
\begin{lem}
Let $h(x)$ be a concave function defined on a subset $\mathcal{D}$ of the real numbers and $h^s(c_i)$ be a supergradient at $c_i$. If $c_i$ and $d_i$ in $\mathcal{D}$ satisfy 
$$\sum_{i=1}^k d_i\, h^s(c_i) \le \sum_{i=1}^k c_i \,h^s(c_i)\ ,$$ 
then 
$$\sum_{i=1}^k h(d_i) \le \sum_{i=1}^k h(c_i)\ .$$
If the first inequality is strict, so is the second.
\label{lem:conv}
\end{lem}
\begin{proof}
See ~\cite{aftabwacv15}.
\end{proof}

Note that Lemma~\ref{lem:conv} only considers the case where the function $h$ is the same for all the elements in the sum. This is in contrast with our definition of the cost in Eq.~\ref{eqn:hcost}, where we want to allow $h$ to be indexed on $i$. To handle this more general scenario, we introduce the following lemma.
\begin{lem}
Given a set $\mathcal{X}$ and functions $f_{i}: \mathcal{X} \rightarrow \mathcal{D}$ and concave functions $h_i: \mathcal{D} \rightarrow {\rm I\!R}$, where $\mathcal{D} \subseteq {\rm I\!R}$, such that,
$$
\sum_{i=1}^k w_{i}^{t}\, f_{i}(\mathbf{x}^{t+1}) \le \sum_{i=1}^k w_{i}^{t}\, f_{i}(\mathbf{x}^{t})\ ,
$$
where $w_i^t = h_i^s(f_{i}(\mathbf{x}^t))$, and $\mathbf{x}^t$ is the estimate of $\mathbf{x}$ at iteration $t$. Then,
$$
\sum_{i=1}^k h_i \circ f_{i}(\mathbf{x}^{t+1}) \le \sum_{i=1}^k h_i \circ f_{i}(\mathbf{x}^{t})\ .
$$
If the first inequality is strict, so is the second.
\label{lem:dec}
\end{lem}
\begin{proof}
Let us define $c_i = f_{i}(\mathbf{x}^t)$ and $d_i = f_{i}(\mathbf{x}^{t+1})$. Since $h_i^s$ is a supergradient, 
$$
h_i(d_i) \le h_i(c_i) + (d_i - c_i) \,h_i^s(c_i)\ ,
$$
for all $i$. Summing over $i$ gives, 
$$
\sum_{i=1}^k h_i(d_i) \le \sum_{i=1}^k h_i(c_i) + \sum_{i=1}^k (d_i - c_i)\, h_i^s(c_i)\ .
$$ 
The sum $\sum_{i=1}^k (d_i - c_i) \,h^s(c_i) = \sum_{i=1}^k w_{i}^{t}\, f_{i}(\mathbf{x}^{t+1}) - \sum_{i=1}^k w_{i}^{t}\, f_{i}(\mathbf{x}^{t})$ is non-positive by hypothesis, which completes the proof.
\end{proof}
It is important to note that this lemma holds for discrete subsets $\mathcal{D}$, as well as continuous ones, and that the functions $h_i$ do not need to be differentiable. 

Therefore, for concave functions $h_i$, by choosing the supergradients of $h_i$ as weights at each iteration, we can minimize the objective function $C_h(\mathbf{x})$ of Eq.~\ref{eqn:hcost} by iteratively minimizing the cost $C_w(\mathbf{x})$ of Eq.~\ref{eqn:wcost}. This general procedure is summarized in Algorithm~\ref{alg:ir}.

\begin{algorithm}[t!]
\caption{Iteratively reweighted minimization}
\label{alg:ir}
\begin{algorithmic}
\State $C_h(\mathbf{x}) \gets \sum_{i=1}^k h_i \circ f_i(\mathbf{x})$ \Comment {Concave functions $h_i$}
\State Initialize $\mathbf{x}$
\Repeat 
\State $w_i^t \gets h_i^s(f_{i}(\mathbf{x}^t))$
\State $\mathbf{x}^{t+1} \gets \underset{\mathbf{x}} {\operatorname{arg\,min}} \sum_{i=1}^k w_i^t\, f_i(\mathbf{x})$
\Until convergence of $C_h(\mathbf{x})$\\
\Return $\mathbf{x}^{t+1}$
\end{algorithmic}
\end{algorithm}

Algorithm \ref{alg:ir} is applicable to any minimization problem, as long as the objective function takes the form of Eq.~\ref{eqn:hcost} with concave functions $h_i$. Furthermore, to minimize the surrogate cost of  Eq.~\ref{eqn:wcost}, any algorithm (either exact or approximate) can be used, as long as it decreases this cost.

\section{An iteratively reweighted scheme for MRFs}
\label{sec:irmrf}

Recall that our goal is to tackle the problem of MRF energy minimization. Here, we show how this can be achieved 
by exploiting Algorithm~\ref{alg:ir}.

To this end, let $\mathcal{V}$ be the set of vertices (or nodes) of an MRF, \eg, corresponding to the pixels or superpixels in an image, and $\mathcal{L}$ be a finite set of labels. Furthermore, let $\mathbf{x} = [x_1, \cdots, x_n]^T \;, \; n = |\mathcal{V}|\;,\; x_p \in \mathcal{L}$ denote the vector that assigns one label to each node. Finding the best label assignment for each node in an MRF then translates to minimizing the energy of the MRF.
In general, the energy of an MRF can be expressed as
\begin{equation}
E(\mathbf{x}) = \sum_{i=1}^{|\mathcal{C}|} \theta_{i}(\mathbf{x}_i)\ ,
\label{eqn:tmrf}
\end{equation}
where $\mathcal{C}$ is the set of cliques in the graph (i.e., the groups of connected nodes), $\mathbf{x}_i$ represents the set of variables corresponding to the nodes in clique $i$, and $\theta_{i} : \mathcal{L}^{|\mathbf{x}_i|} \rightarrow {\rm I\!R}$ is the energy (or potential) function associated with clique $i$.

Let us now assume that each potential function can be written as
\begin{equation}
\theta_i(\mathbf{x}_i) = h_i \circ f_i(\mathbf{x}_i)\ ,
\label{eqn:irpot}
\end{equation}
where $h_i$ is a concave function and $f_i$ an arbitrary one. This lets us rewrite the MRF energy of Eq.~\ref{eqn:tmrf} as
\begin{equation}
E(\mathbf{x}) = \sum_{i=1}^{|\mathcal{C}|} h_i \circ f_i(\mathbf{x}_i)\ ,
\label{eqn:hmrf}
\end{equation}
which has the form of Eq.~\ref{eqn:hcost}\footnote{Note that $f_i(\mathbf{x}_i)$ can be equivalently written as $f_i(\mathbf{x})$, where the variables $x_p \notin \mathbf{x}_i$ (i.e., not in clique $i$) simply have no effect on the function.}. Therefore, we can employ Algorithm \ref{alg:ir} to minimize $E(\mathbf{x})$, and iteratively minimize the surrogate energy
\begin{equation}
\tilde{E}(\mathbf{x}) = \sum_{i=1}^{|\mathcal{C}|} w_i f_i(\mathbf{x}_i)\ ,
\label{eqn:wmrf}
\end{equation}
with weight $w_i$ taken as the supergradient of $h_i$ evaluated at the previous estimate of $\mathbf{x}$.

It is important to note, however, that for this algorithm to be effective, the minimization of $\tilde{E}(\mathbf{x})$ at each iteration must be relatively easy, and at least guarantee that the surrogate energy decreases. Furthermore, while in practice any existing MRF energy minimization algorithm (either exact or approximate) can be utilized to minimize $\tilde{E}(\mathbf{x})$, the quality of the overall solution found by our algorithm may vary accordingly. In the next section, we discuss a special case of this general MRF energy minimization algorithm, which, as depicted in our experiments, is effective in many scenarios.

\begin{algorithm*}[t]
\caption{Iteratively Reweighted Graph Cut (IRGC)}
\label{alg:irg}
\begin{algorithmic}
\State $E(\mathbf{x}) \gets \sum_{p \in \mathcal{V}} f_p(x_p) + \sum_{(p,q) \in \mathcal{N}} h_b \circ g\left(\left|x_p - x_q\right|\right)$ \Comment{Convex function $g$ and non-decreasing concave function $h_b$}
\State $w_{pq}^0 \gets 0.5$ \Comment{Initialize the weights}
\Repeat 
\State \textbf{if} $t \neq 0$ \textbf{then} $w_{pq}^{t} \gets h^s_b\left(g\left(\left|x_p^t - x_q^t\right|\right)\right)$ \textbf{end if} \Comment{Update the weights except at the first iteration}
\State $\mathbf{x}^{t+1} \gets \underset{\mathbf{x}} {\operatorname{arg\,min}} \sum_{p \in \mathcal{V}} f_p(x_p) + \sum_{(p,q) \in \mathcal{N}} w_{pq}^t\, g(|x_p - x_q|)$ \Comment{Minimize using the multi-label graph cut}
\Until {$E(\mathbf{x}^{t+1}) = E(\mathbf{x}^t)$} \Comment{Convergence of $E(\mathbf{x})$}\\
\Return $\mathbf{x}^{t+1}$
\end{algorithmic}
\end{algorithm*}

\section{Iteratively reweighted graph cut}

In this section, we introduce an iterative algorithm for the case of multi-label MRFs with pairwise node interactions. In particular, we propose to make use of the multi-label graph cut of~\cite{ishikawa2003exact} at each iteration of our algorithm. The multi-label graph cut yields an optimal solution under the following two conditions: 1) the label set must be ordered; 2) the pairwise potential must be a convex function of the label difference. In practice, such convex priors have limited power due to their poor ability to model noise. In contrast, while still relying on the first condition, our algorithm allows us to generalize to non-convex priors, and in particular to robust norms that have proven effective in computer vision.
 
\subsection{MRF with pairwise interactions}
\label{sec:pair}
In an MRF with pairwise node interactions, the energy can be expressed as
\begin{equation}
E(\mathbf{x}) = \sum_{p \in \mathcal{V}} \theta^u_p(x_p) + \sum_{(p, q) \in \mathcal{N}} \theta^b_{pq}(x_p, x_q)\ ,
\label{eqn:mrfp}
\end{equation}
where $\theta^u$ and $\theta^b$ denote the unary potentials (i.e., data cost) and pairwise potential (i.e., interaction cost), respectively, and $\mathcal{N}$ is the set of edges in the graph, \eg, encoding a 4-connected or 8-connected grid over the image pixels.

As discussed in Section~\ref{sec:irmrf}, to be able to make use of Algorithm~\ref{alg:ir}, we need to have potential functions of the form given in Eq.~\ref{eqn:irpot}. Under this assumption, we can then rewrite the energy of Eq.~\ref{eqn:mrfp} as
\begin{equation}
E(\mathbf{x}) =\sum_{p \in \mathcal{V}} h_u \circ f_p(x_p) + \sum_{(p, q) \in \mathcal{N}} h_b \circ f_{pq}(x_p,x_q)\ ,
\label{eqn:hmrf1}
\end{equation}
where $h_u$ and $h_b$ are concave functions.

Following Algorithm~\ref{alg:ir}, we minimize this energy by iteratively minimizing a surrogate energy of the form (at iteration $t+1$)
\begin{eqnarray}
\tilde{E}(\mathbf{x}) &=& \sum_{p \in \mathcal{V}} h^s_u\hspace{-0.1cm}\left(f_p(x_p^t)\right) f_p(x_p)\\ &+& \sum_{(p, q) \in \mathcal{N}} h^s_b\hspace{-0.1cm}\left(f_{pq}(x^t_p,x^t_q)\right) f_{pq}(x_p,x_q)\ , \nonumber
\label{eqn:wmrf1}
\end{eqnarray}
where $h^s_u$ and $h^s_b$ are the supergradients of $h_u$ and $h_b$, respectively, and $x_p^t$ denote the estimate of $x_p$ at the previous iteration.

Since our goal is to employ the multi-label graph cut algorithm~\cite{ishikawa2003exact} to minimize $\tilde{E}(\mathbf{x})$, we need to define the different functions $h_u$, $h_b$, $f_p$ and $f_{pq}$ so as to satisfy the requirements of this algorithm. To this end, for the unary potential, we choose $h_u$ to be the identity function. That is,
\begin{equation}
\theta^u_p(x_p) = h_u \circ f_p(x_p) = f_p(x_p)\ .
\end{equation}
This implies that no reweighting is required for the unary potentials, since the supergradient of $h_u$ is always 1. The multi-label graph cut having no specific requirement on the form of the data term, $f_p$ can be any arbitrary function.

In contrast, for the pairwise potentials, the multi-label graph cut requires $f_{pq}$ to be a convex function of the label difference. That is, for a convex function $g$ defined on a subset of ${\rm I\!R}$,
\begin{equation}
f_{pq}(x_p,x_q) = g(|x_p - x_q|)\ .
\label{eqn:ish0}
\end{equation}
In addition, because the energy $\tilde{E}(\mathbf{x})$ depends on a {\it weighted} sum of pairwise terms, we need the weights to satisfy some conditions. More precisely, to be able to use the max-flow algorithm within the multi-label graph cut, the weights need to be non-negative. Since these weights are computed from the supergradient of $h_b$, this translates into a requirement for $h_b$ to be non-decreasing. Note that, in the context of smoothness potentials in an MRF, this requirement comes at virtually no cost, since we hardly expect the potentials to decrease as the label difference increases.

Under these conditions, the surrogate energy to be minimized by the multi-label graph cut at each iteration of our algorithm can be written as
\begin{equation}
\tilde{E}({\mathbf{x}}) = \sum_{p \in \mathcal{V}} f_p(x_p) + \sum_{(p,q) \in \mathcal{N}} w^t_{pq} g(|x_p - x_q|)\ ,
\label{eqn:wmrf1}
\end{equation}
where $g$ is a convex function, and $w^t_{pq} = h^s_b\hspace{-0.1cm}\left(f_{pq}(x^t_p,x^t_q)\right)$, with $h_b$ a concave, non-decreasing function. In the first iteration, we set the weights $w_i^0$ to some constant value\footnote{We set $w_{pq}^0 = \epsilon$, where $0 \le \epsilon \le 1$, so that the effect of the edge terms is smaller for the first estimate. In our experiments, we found $\epsilon=0.5$ to work well and thus always use this value.} to make the algorithm independent of any initial estimate $\mathbf{x}^0$. Our overall Iteratively Reweighted Graph Cut (IRGC) algorithm is summarized in Algorithm \ref{alg:irg}. 

\paragraph{Multi-label graph cut.}
Here, we briefly describe the multi-label graph cut algorithm employed at each iteration of Algorithm \ref{alg:irg}. Note that, although equivalent, our graph construction slightly differs from that of~\cite{ishikawa2003exact}.


The multi-label graph cut works by constructing a graph such as the one illustrated in Fig.~\ref{fig:ishikawa} for two neighboring nodes $p$ and $q$ of the original MRF, and applying graph cut to this graph. More specifically, 
let the label set $\mathcal{L} = \{0, 1, 2, \cdots, l-1\}$. For each node $p \in \mathcal{V}$, the new graph contains $l-1$ nodes denoted by $p_0, p_1, \cdots, p_{l-2}$. Furthermore, two additional nodes, the start and end nodes denoted by $0$ and $1$, are included in the graph.  

\begin{figure}[t]
\begin{center}
\includegraphics[width=0.9\linewidth, trim=2.9cm 7.9cm 9.1cm 4.7cm, clip=true, page=2]{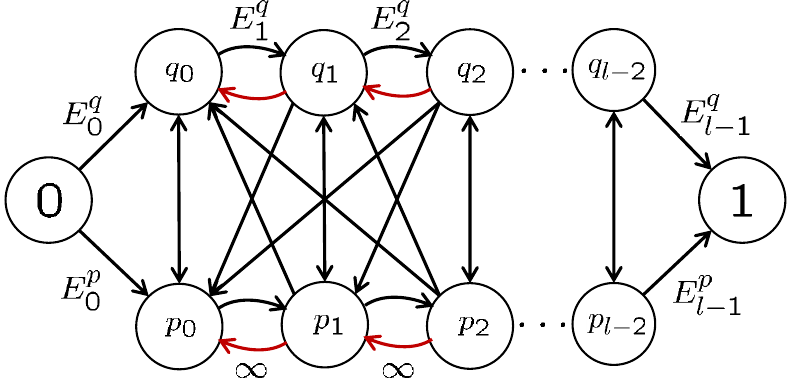}	
\end{center}
\vspace{-0.4cm}
\caption{\em Multi-label graph construction. Constraint edges are highlighted in red.}
\label{fig:ishikawa}
\vspace{-0.2cm}
\end{figure}

For each node $p$ and each $i \in {L}$, there is a directed horizontal edge from $p_{i-1}$ to $p_{i}$, where $p_{-1}$ and $p_{l-1}$ denote the start and end nodes, respectively. Additionally, the graph contains constraint edges with infinite weights going in the opposite direction. Since the graph cut algorithm partitions the set of nodes into two disjoint subsets $\mathcal{V}_0$ and $\mathcal{V}_1$ where $0 \in \mathcal{V}_0$ and $1 \in \mathcal{V}_1$, the constraint edges ensure that for each $p \in \mathcal{V}$ exactly one horizontal edge is present in the minimum cut. If an edge $p_{i-1} \rightarrow p_{i}$ is in the minimum cut, then it can be interpreted as node $p$ taking the label $i$. Therefore, the horizontal edges directly represent the unary potentials and are given a weight $E_i^p = f_p(i)$.

Furthermore, for any $(p,q) \in \mathcal{N}$ and for each $i,j \in \{0, 1, \cdots, l-2\}$ there is a directed edge from $p_i$ to $q_j$ and vice versa. The weight of the edge $p_i \rightarrow q_j$ is defined as
\begin{equation}
E_{ij}^{pq} = \left\{ \begin{array}{lll}
	0 & \mbox{if $i < j$} \\
           \frac{w_{pq}^t}{2}\, g''(|i-j|) & \mbox{if $i=j$} \\
	 w_{pq}^t\,g''(|i-j|) & \mbox{if $i>j$} \ , \end{array} \right.
\label{eqn:ishbew}
\end{equation}
where $g''(|z|) = g(|z+1|) + g(|z-1|) - 2\,g(|z|)$, which is positive for a convex function $g$.

In our scenario, with our condition that $w_{pq}^t$ be non-negative, the multi-label graph contains no negative edges. Therefore, the global minimum of the corresponding energy can be found in polynomial time using the max-flow algorithm. Note that for a sparsely connected graph, \eg, 4 or 8-connected neighbourhood, the memory requirement of a general multi-label graph is $\mathcal{O}(|\mathcal{V}|\cdot|\mathcal{L}|^2)$. However if $g$ is linear, then only the vertical edges $p_i \rightarrow q_i$ will have non-zero weights and the memory required drops to $\mathcal{O}(|\mathcal{V}|\cdot|\mathcal{L}|)$.

\subsection{Choice of functions $g$ and $h_b$}

While, in Section~\ref{sec:pair}, we have defined conditions on the functions $g$ and $h_b$ (i.e., $g$ convex and $h_b$ concave, non-decreasing) for our algorithm to be applicable with the multi-label graph cut, these conditions still leave us a lot of freedom in the actual choice of these functions. Here, we discuss several such choices, with special considerations on the memory requirement of the resulting algorithm.

In the context of computer vision problems with ordered label sets, \eg, stereo and inpainting, it is often important to make use of robust estimators as pairwise potentials to better account for discontinuities, or outliers. Many such robust estimators belong to the family of functions with a single inflection point in ${\rm I\!R}^+$. In other words, they can be generally defined as non-decreasing functions $\theta_{pq}(|x_p - x_q|)$\footnote{Note that, here, since we only consider pairwise terms, we dropped the superscript $b$ in $\theta^b$ for ease of notation.}, such that for a given $\lambda \ge 0$, $\theta(z)$ is convex if $z \le \lambda$, and concave otherwise. Such functions include the truncated linear, the truncated quadratic and the Cauchy function $\theta(z) = \lambda^2/2\,\log\left(1 + (z/\lambda)^2\right)$~\cite{hartley2003multiple}. Note that any convex, or concave function on ${\rm I\!R}^+$ also belongs to this family.

\begin{figure*}[t]
\begin{center}
\begin{subfigure}{.33\textwidth}
	\includegraphics[width=0.9\linewidth]{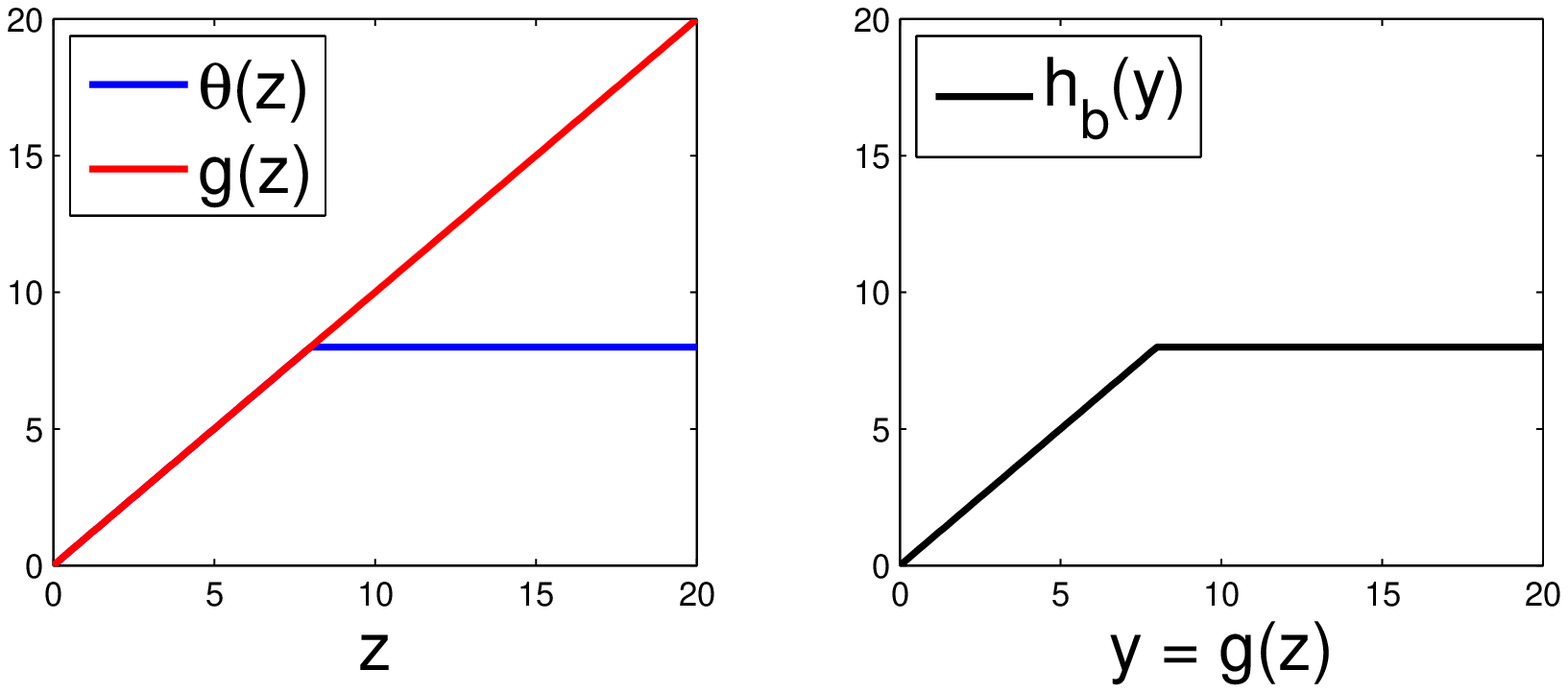}
	\caption{$\theta$ - Truncated linear}
	\label{sfig:ghtl1}
\end{subfigure}%
\begin{subfigure}{.33\textwidth}
	\includegraphics[width=0.9\linewidth]{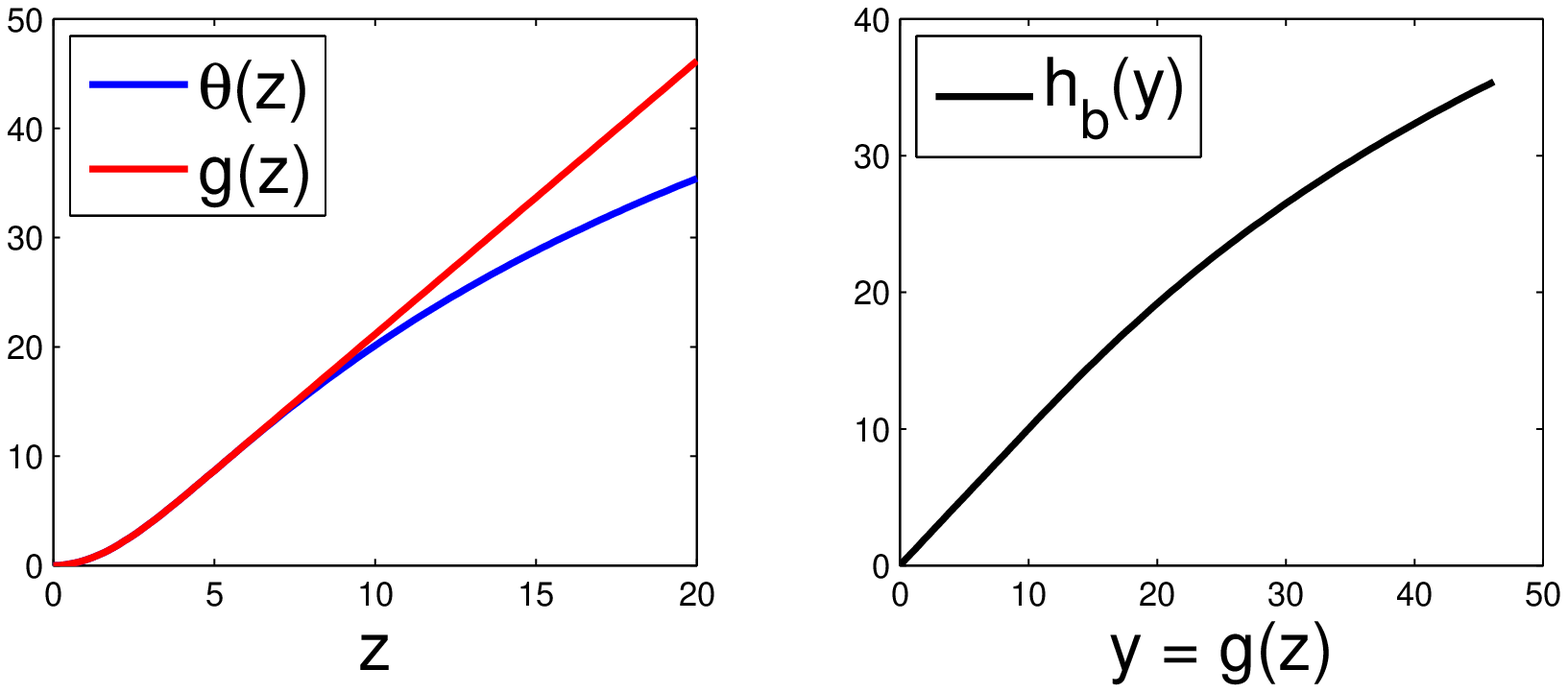}
	\caption{$\theta$ - Cauchy function}
	\label{sfig:ghcau}
\end{subfigure}%
\begin{subfigure}{.33\textwidth}
	\includegraphics[width=0.9\linewidth]{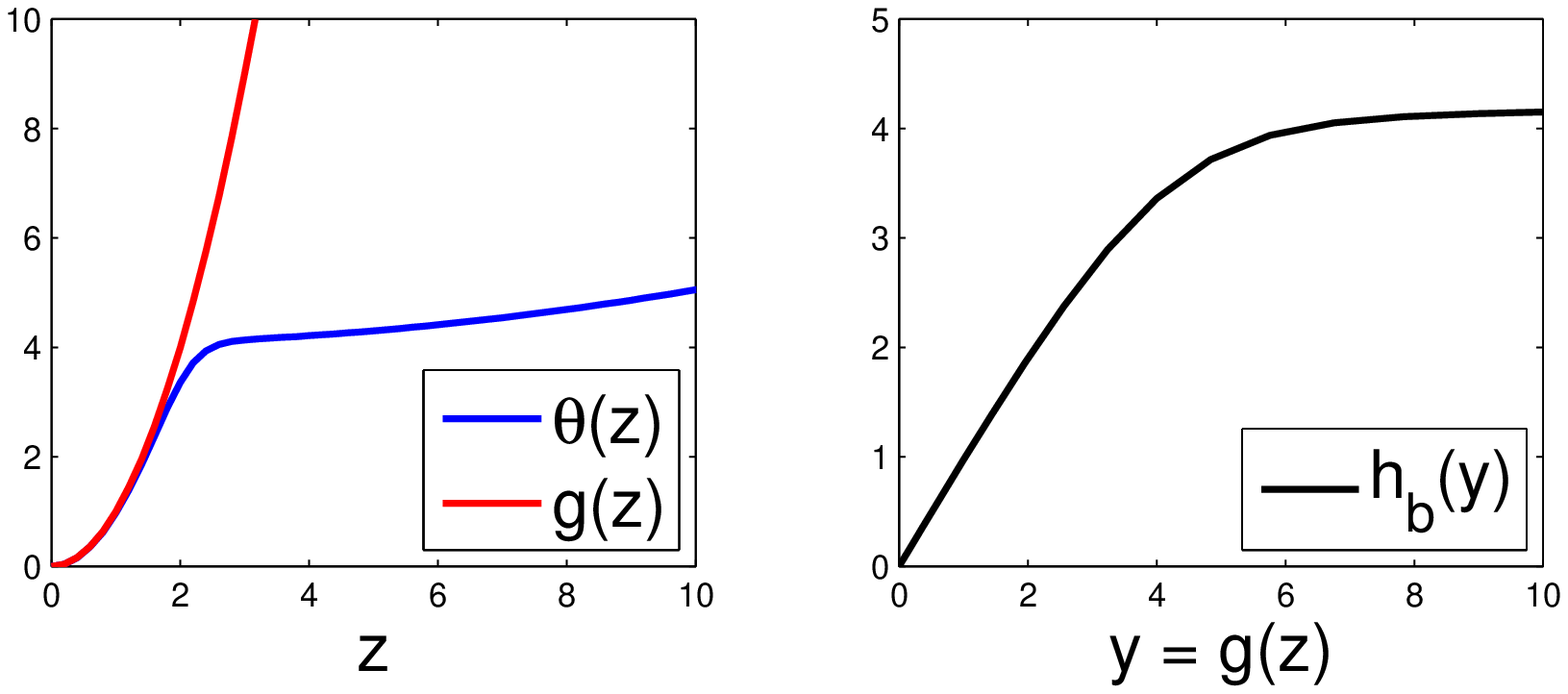}
	\caption{$\theta$ - Corrupted Gaussian}
	\label{sfig:ghgau}
\end{subfigure}
\end{center}
\vspace{-0.4cm}
 \caption{\em Plots of $\theta$, $g$ and $h_b$ with $\theta(z) = h_b \circ g(z)$, when $\theta$ is \textbf{(a)} the truncated linear, \textbf{(b)} the Cauchy function and \textbf{(c)} the corrupted Gaussian. Here $g$ is convex and $h_b$ is concave. In \textbf{(a)} and \textbf{(b)}, the functions $g$ and $h_b$ are derived from Table \ref{tab:gh}. In \textbf{(c)}, $g(z) = z^2$ and $h_b(y) = \theta(\sqrt{y})$.}
\label{fig:gh}
\vspace{-0.2cm}
\end{figure*}

\begin{table}[t]
\begin{center}
\begin{tabular}{>{\centering\arraybackslash}m{0.86cm}|>{\centering\arraybackslash}m{1.9cm}||>{\centering\arraybackslash}m{1.31cm}|>{\centering\arraybackslash}m{2.4cm}}
 & $y=g(z)$ & & $h_b(y)$ \\
\hline
$z \le \lambda$ & $\theta(z)$ & $y \le \theta(\lambda)$ & $y$\\
\hline
$z \ge \lambda$ & $\begin{array}{c} \theta'(\lambda)(z-\lambda) \\ + \theta(\lambda) \end{array}$ & $y \ge \theta(\lambda)$ & $\theta\left(\frac{y+\lambda\,\theta'(\lambda)-\theta(\lambda)}{\theta'(\lambda)}\right)$\\
\end{tabular}
\end{center}
\vspace{-0.4cm}
\caption{\em Functions $g$ and $h_b$ with $\theta(z) = h_b \circ g(z)$, where $\theta(z)$ is convex if $z \le \lambda$ and concave otherwise. It can easily be verified that $g$ is convex and $h_b$ is concave, as well as that both functions are non-decreasing, because $\theta$ is non-decreasing. Here $\theta'(\lambda)$ is the derivative of $\theta$ at $\lambda$, or its left derivative $\theta'(\lambda^-)$ if $\theta$ is not differentiable at $\lambda$. See Fig.~\ref{fig:gh}(a-b) for example plots.}
\label{tab:gh}
\vspace{-0.2cm}
\end{table}

For a given such function $\theta$, according to our algorithm, we need to write
\begin{equation}
\theta(z) = h_b \circ g(z)\ ,
\label{eq:phidecomp}
\end{equation}
with a concave function $h_b$ and a convex function $g$. Note that the multi-label graph structure is determined by the function $g$. Therefore, to make the graph as sparse as possible, and thus limit the required memory, we need to choose $g$ such that the second order difference $g''(z)$ is zero for as many values $z$ as possible. Table \ref{tab:gh} gives the functions $g$ and $h_b$ such that $g''(z)$ is zero $\forall z \ge \lambda$ while satisfying Eq.~\ref{eq:phidecomp} and the necessary conditions on $h_b$ and $g$. Fig.~\ref{fig:gh}(a-b) provide the plots corresponding to the truncated linear and Cauchy function.
For a function $g$ derived according to Table \ref{tab:gh}, the memory requirement of the multi-label graph is $\mathcal{O}(|\mathcal{V}|\cdot|\mathcal{L}|\cdot\lambda)$. 

Note that our method is not limited to the family of functions described above. As an example, we consider the case of another robust estimator, the corrupted Gaussian function $\theta_G(z) = -\log(\alpha\exp(-z^2)+(1-\alpha)\exp(-z^2/\beta^2)/\beta)$~\cite{hartley2003multiple}, which does not follow the definitions of the functions described before. However, since $\theta_G(\sqrt{\cdot})$ is concave, we can minimize $\theta_G(z)$ by choosing $g(z)=z^2$ and $h_b(y) = \theta_G(\sqrt{y})$. The corresponding plots for the corrupted Gaussian are provided in Fig.~\ref{fig:gh}(c).

\subsection{Hybrid strategy}

While our algorithm guarantees that the energy value decreases at each iteration, it remains prone to getting trapped in local minima (with respect to the iterative reweighting scheme). Here, we propose a hybrid optimization strategy that combines IRGC with a different minimization technique, and helps us escape from some of the local minima of the energy.

In particular, here we make use of $\alpha$-expansion~\cite{boykov2001fast} as an additional minimization technique. At each iteration of our algorithm, instead of updating $\mathbf{x}^{t} \rightarrow \mathbf{x}^{t+1}$ in one step, our hybrid strategy performs the following steps:
\begin{enumerate}
\item Update $\mathbf{x}^{t} \rightarrow \mathbf{x}'$ by minimizing the surrogate energy using the multi-label graph cut.
\item Improve the new estimate $\mathbf{x}' \rightarrow \mathbf{x}^{t+1}$ using one pass of $\alpha$-expansion with the true energy, such that $E(\mathbf{x}^{t+1}) \le E(\mathbf{x}')$.
\end{enumerate}

For non-metric pairwise potentials, for which regular $\alpha$-expansion does not apply, we truncate the non-submodular terms as suggested in~\cite{rother2005digital}. Note that this still guarantees that the energy will decrease. We found that the variety in the optimization strategy arising from this additional $\alpha$-expansion step was effective to overcome local minima. Since both the algorithms guarantee to decrease the energy $E(\mathbf{x})$ at each step, our hybrid algorithm also decreases the energy at each iteration. In our experiments, we refer to this hybrid algorithm as IRGC+expansion.

Note that other methods, such as $\alpha$-$\beta$ swap, or any algorithm that guarantees to improve the given estimate can be employed. Alternatively, one could exploit a fusion move strategy~\cite{lempitsky2010fusion} to combine the estimates obtained by the two different algorithms. However, this would come at an additional computation cost, and, we believe, goes beyond the scope of this paper.

\section{Experiments}

We evaluated our algorithm on the problems of stereo correspondence estimation and image inpainting. In those cases, the pairwise potentials typically depend on additional constant weights, and can thus be written as
\begin{equation}
\theta^b_{pq}(x_p,x_q) = \gamma_{pq}\, \phi_{pq}(|x_p - x_q|)\ ,
\end{equation}
where $\gamma_{pq}$ are the constant weights. As long as $\gamma_{pq} \geq 0$, our algorithm is unaffected by these weights, in the sense that we can simply multiply our weights $w_{pq}^t$ by these additional constant weights. Note that since the main purpose of this paper is to evaluate the performance of our algorithm on different MRF energy functions, we used different smoothing costs $\phi(\cdot)$ for different problem instances without tuning the weights $\gamma_{pq}$ for the specific smoothing costs. 

We compare our results with those of $\alpha$-expansion, $\alpha$-$\beta$ swap~\cite{boykov2001fast}, multi-label swap~\cite{veksler2012multi} and TRW-S~\cite{kolmogorov2006convergent}. For $\alpha$-expansion, we used the maxflow algorithm when the pairwise potentials were \textit{metric}, and the QPBOP algorithm~\cite{boros2002pseudo,rother2007optimizing} (denoted as $\alpha$-expansionQ)  otherwise. In the latter case, if a node in the binary problem is unlabeled then the previous label is retained. For our comparison, we used the publicly available implementation of $\alpha$-expansion, $\alpha$-$\beta$ swap, QPBO and TRW-S, and implemented the multi-label swap algorithm as described in~\cite{veksler2012multi}.

All the algorithms were initialized by assigning the label $0$ to all the nodes. For multi-label swap the parameter $t$ was fixed to 2 in all our experiments (see~\cite{veksler2012multi} for details). The energy values presented in the following sections were obtained at convergence of the different algorithms, except for TRW-S which we ran for 100 iterations and chose the best energy value. All our experiments were conducted on a 3.4GHz i7-4770 CPU with 16 GB RAM and no effort was taken to utilize the multiple cores of the processor. 

\vspace{-0.3cm}
\paragraph{Stereo:} \mbox{}\\
Given a pair of rectified images (one left and one right), stereo correspondence estimation aims to find the disparity map, which specifies the horizontal displacement of each pixel between the two images with respect to the left image. For this task, we employed six instances from the Middlebury dataset~\cite{scharstein2002taxonomy, scharstein2003high}: Teddy, Map, Sawtooth, Venus, Cones and Tsukuba. For Tsukuba and Venus, we used the unary potentials of~\cite{szeliski2008comparative}, and for the other cases, those of~\cite{birchfield1998pixel}. The pairwise potentials are summarized in Table \ref{tab:stereo}. Note that we do not explicitly model occlusions, which should be handled by our robust potentials.

\begin{table}[t]
\begin{center}
\begin{tabular}{c|c|c|c}
Problem & $\gamma_{pq}$ & $\phi(\cdot)$ & $\lambda$ \\
\hline
Teddy & $\left\{\begin{array}{ll} 30 & \mbox{if $\nabla_{pq} \le 10$} \\ 10 & \mbox{otherwise} \end{array} \right.$ & \multirow{2}{*}{\parbox{1.5cm}{Truncated linear}} & 8 \\
Map & 4 & & 6 \\
\hline
Sawtooth & 20 & \multirow{2}{*}{\parbox{1.5cm}{Truncated quadratic}} & 3 \\
Venus & 50 &  & 3 \\
\hline
Cones & 10 & \multirow{3}{*}{\parbox{1.5cm}{Cauchy function}} & 8 \\
Tsukuba & $\left\{\begin{array}{ll} 40 & \mbox{if $\nabla_{pq} \le 8$} \\ 20 & \mbox{otherwise} \end{array} \right.$ & & 2 \\ 
\end{tabular}
\end{center}
\vspace{-0.4cm}
\caption{\em Pairwise potential $\theta^b_{pq}(x_p, x_q) = \gamma_{pq}\,\phi(|x_p - x_q|)$ used for the stereo problems. Here $\phi(z)$ is convex if $z \le \lambda$ and concave otherwise, and $\nabla_{pq}$ denotes the absolute intensity difference between the pixels $p$ and $q$ in the left image.}
\label{tab:stereo}
\end{table}

The final energies and execution times corresponding to the stereo problems are summarized in Table~\ref{tab:energys}. The disparity maps found using our IRGC+expansion algorithm and energy vs time plots of the algorithms for some of the problems are shown in Fig.~\ref{fig:stereo} and Fig.~\ref{fig:evst}(a-c), respectively. Note that IRGC+expansion yields the lowest energy in most cases, and when it does not, an energy that is virtually the same as the lowest one.

\begin{table*}[t]
\begin{center}
\begin{tabular}{>{\centering\arraybackslash}m{2.4cm}|>{\centering\arraybackslash}m{1cm}>{\centering\arraybackslash}m{0.5cm}|>{\centering\arraybackslash}m{1cm}>{\centering\arraybackslash}m{0.5cm}|>{\centering\arraybackslash}m{1cm}>{\centering\arraybackslash}m{0.5cm}|>{\centering\arraybackslash}m{1cm}>{\centering\arraybackslash}m{0.5cm}|>{\centering\arraybackslash}m{1cm}>{\centering\arraybackslash}m{0.5cm}|>{\centering\arraybackslash}m{1cm}>{\centering\arraybackslash}m{0.5cm}}
\multirow{2}{*}{Algorithm} & \multicolumn{2}{c}{Teddy} & \multicolumn{2}{c|}{Map} & \multicolumn{2}{c}{Sawtooth} & \multicolumn{2}{c|}{Venus} & \multicolumn{2}{c}{Cones} & \multicolumn{2}{c}{Tsukuba} \\
	& E[$10^3$] & T[s] & E[$10^3$] & T[s] & E[$10^3$] & T[s] & E[$10^3$] & T[s] & E[$10^3$] & T[s] & E[$10^3$] & T[s]\\
\hline
$\alpha$-$\beta$ swap&2708.1&    87& 149.5&     5&1079.5&    18&3219.7&    21&4489.9&   495& 409.1&    14\\
$\alpha$-expansion(Q)&2664.6&    48& 144.4&     4&1067.5&    26&3201.5&    28&2480.7&   453& 403.3&    15\\
Multi-label swap&5502.3&   235& 470.6&    12&1660.6&   102&5740.1&   162&-&-&-&-\\
TRW-S&2652.7&   362& \textbf{143.0}&    41&1038.8&    55&3098.6&    55&2304.8&   337& \textbf{395.8}&    20\\
IRGC&2687.8&   104& 144.0&    14&1042.1&   129&3081.4&    70&\textbf{2301.4}&   655& 397.3&    30\\
IRGC+expansion&\textbf{2650.3}&    77& 143.2&     8&\textbf{1034.9}&    44&\textbf{3078.8}&    39&\textbf{2301.4}&   336& 396.1&    23\\
\end{tabular}
\end{center}
\vspace{-0.4cm}
\caption{\em Comparison of the minimum energies (E) and execution times (T) for stereo problems. IRGC+expansion found the lowest energy for most of the problems and virtually the same energy as TRW-S otherwise. For the truncated linear prior (Teddy and Map) IRGC+expansion was 5 times faster than TRW-S. For the truncated quadratic and the Cauchy prior IRGC outperformed all other graph-cut-based algorithms and found a lower energy than TRW-S for Venus and Cones.}
\label{tab:energys}
\end{table*}

\begin{figure}[t]
\begin{center}

\begin{subfigure}{.16\textwidth}
	\includegraphics[height=2.15cm]{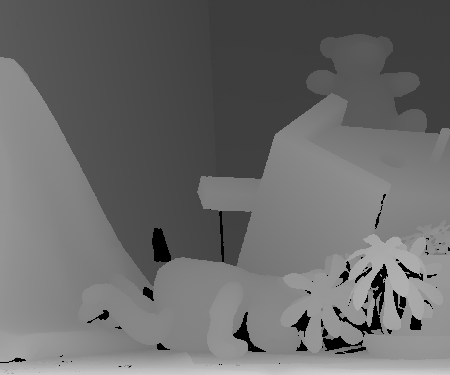}\\[0.5em]
	\includegraphics[height=2.15cm]{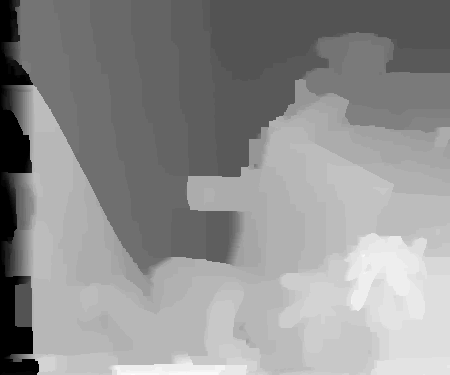}
	\caption{Teddy,\\Tr. linear}
\end{subfigure}%
\begin{subfigure}{.152\textwidth}
	\includegraphics[height=2.15cm]{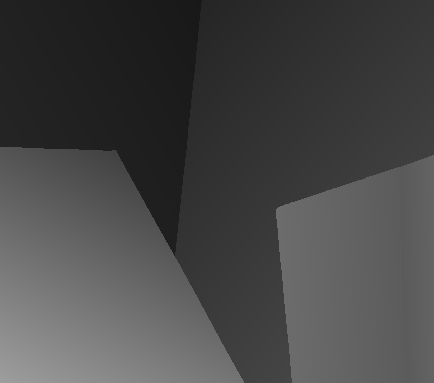}\\[0.5em]
	\includegraphics[height=2.15cm]{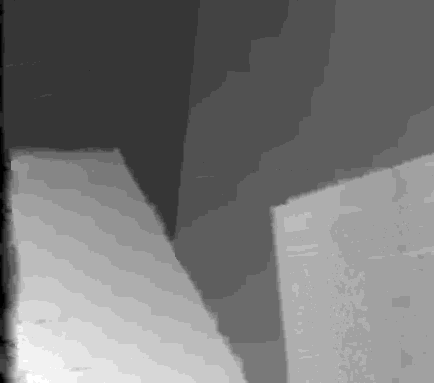}
	\caption{Venus,\\Tr. quad.}
\end{subfigure}%
\begin{subfigure}{.18\textwidth}
	\includegraphics[height=2.15cm]{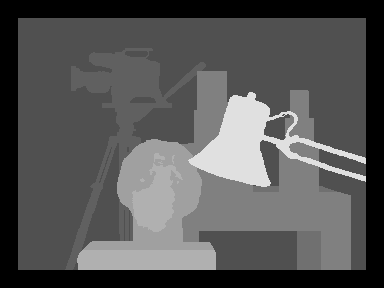}\\[0.5em]
	\includegraphics[height=2.15cm]{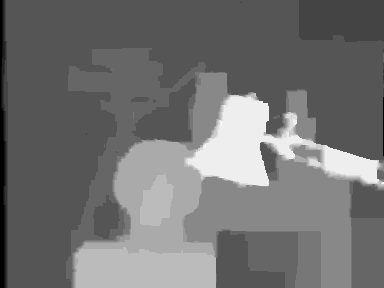}
	\caption{Tsukuba,\\Cauchy}
\end{subfigure}%
\end{center}
\vspace{-0.4cm}
   \caption{\em Disparity maps obtained with our IRGC+expansion algorithm. The corresponding ground-truth is shown above each disparity map.}
\label{fig:stereo}
\end{figure}

To illustrate the fact that our algorithm can also exploit priors that are not first convex and then concave, we employed a corrupted Gaussian pairwise potential (with parameters $\alpha=0.75$ and $\beta=50$) on the Tsukuba stereo pair. The results shown in Table~\ref{tab:cgau} demonstrate that our IRGC+expansion algorithm also outperforms the baselines in this case.

\begin{table}[t]
\begin{center}
\begin{tabular}{c|cc}
\multirow{2}{*}{Algorithm} & \multicolumn{2}{c}{Tsukuba} \\
	& E[$10^3$] & T[s]\\
\hline
$\alpha$-$\beta$ swap& 568.3&    18\\
$\alpha$-expansionQ& 555.2&    15\\
TRW-S& 550.0&    22\\
IRGC& 614.3&    72\\
IRGC+expansion& \textbf{549.3}&    53\\
\end{tabular}
\end{center}
\vspace{-0.4cm}
\caption{\em Comparison of the minimum energies (E) and execution times (T) on Tsukuba with a corrupted Gaussian prior. While IRGC was trapped in a local minimum, IRGC+expansion found the lowest energy.}
\label{tab:cgau}
\end{table}

\begin{figure*}[t]
\begin{center}
\begin{subfigure}{.25\textwidth}
	\includegraphics[width=0.98\linewidth]{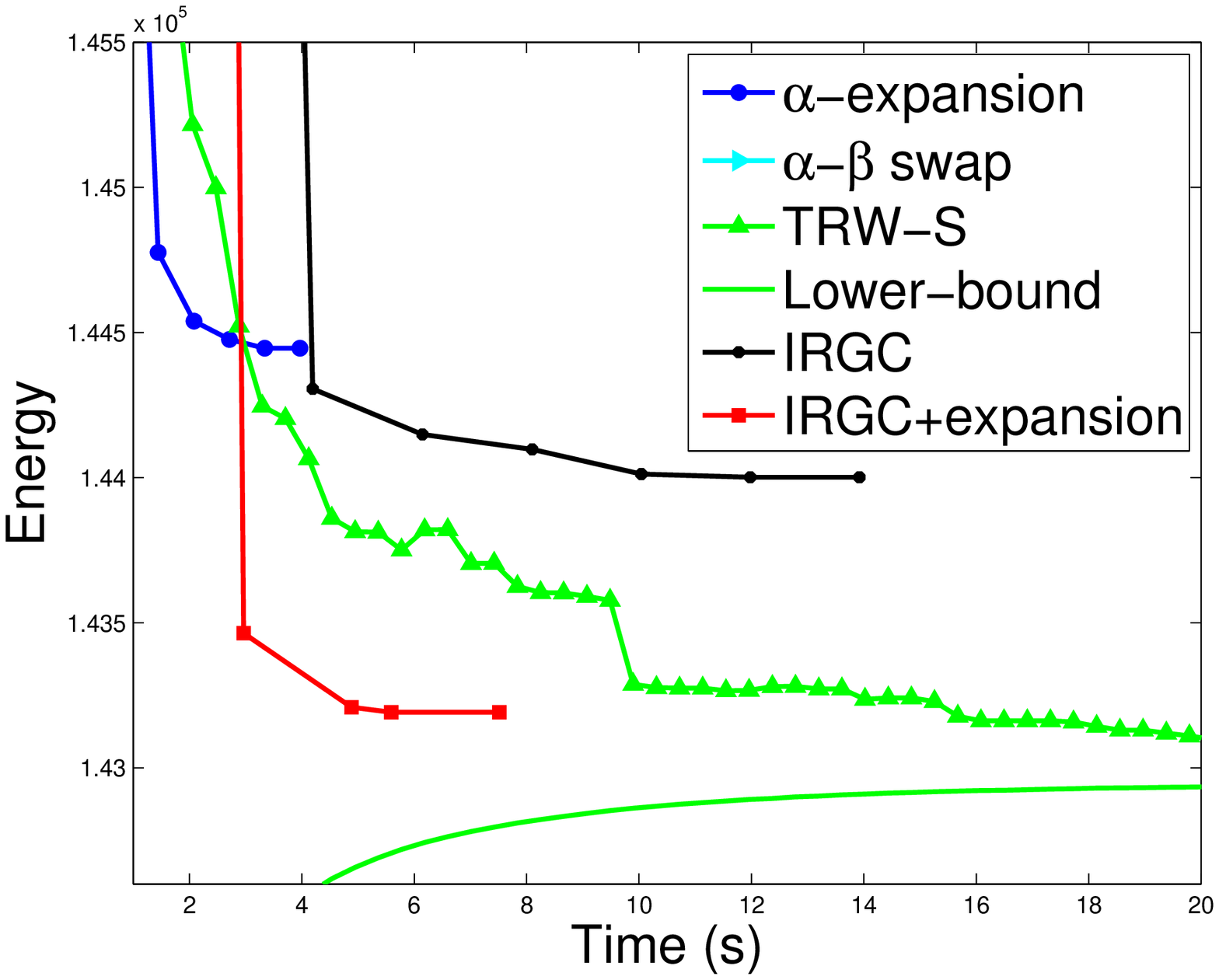}
	\caption{Map,\\Trunc. linear}
\end{subfigure}%
\begin{subfigure}{.25\textwidth}
	\includegraphics[width=0.98\linewidth]{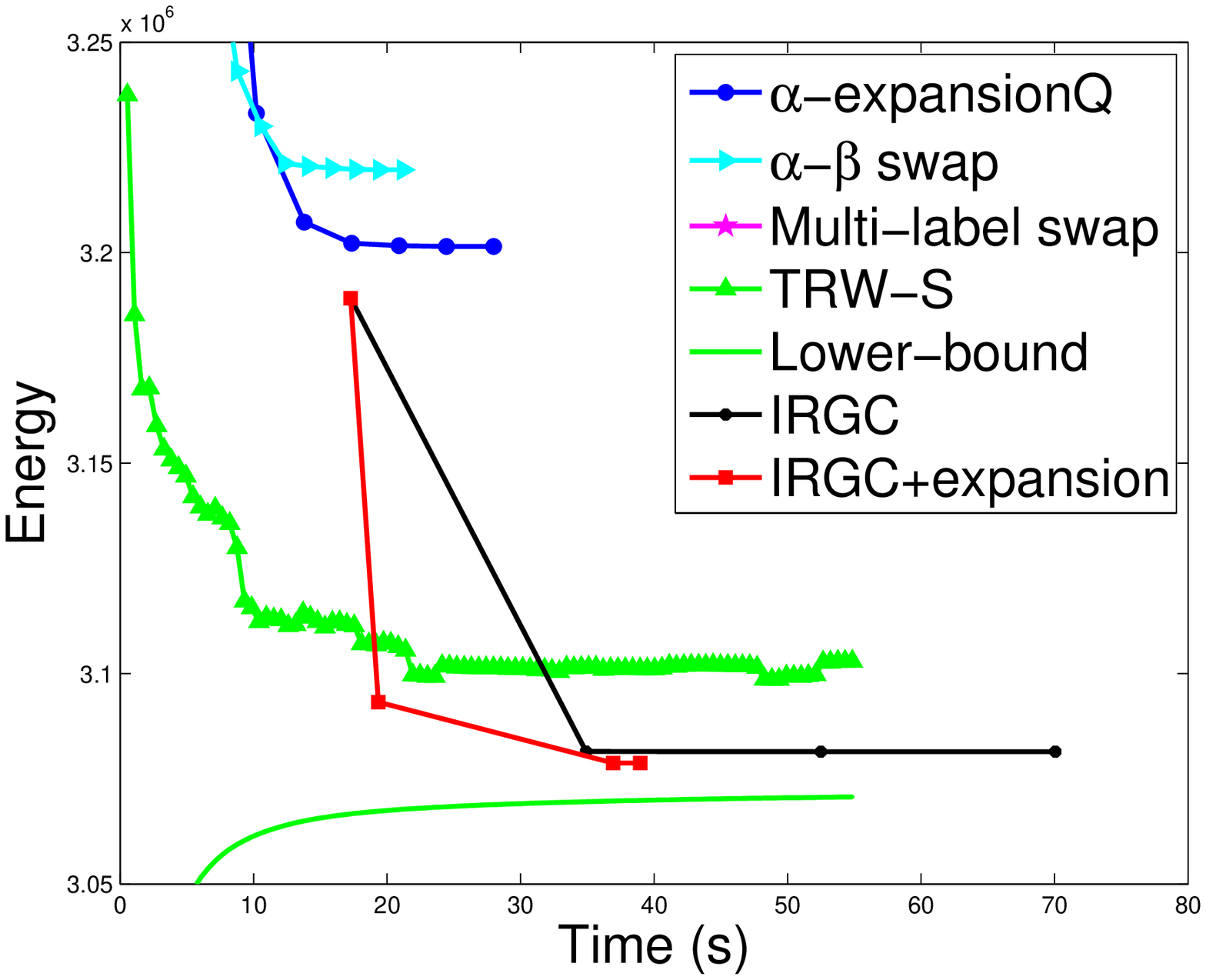}
	\caption{Venus,\\Trunc. quadratic}
\end{subfigure}%
\begin{subfigure}{.25\textwidth}
	\includegraphics[width=0.98\linewidth]{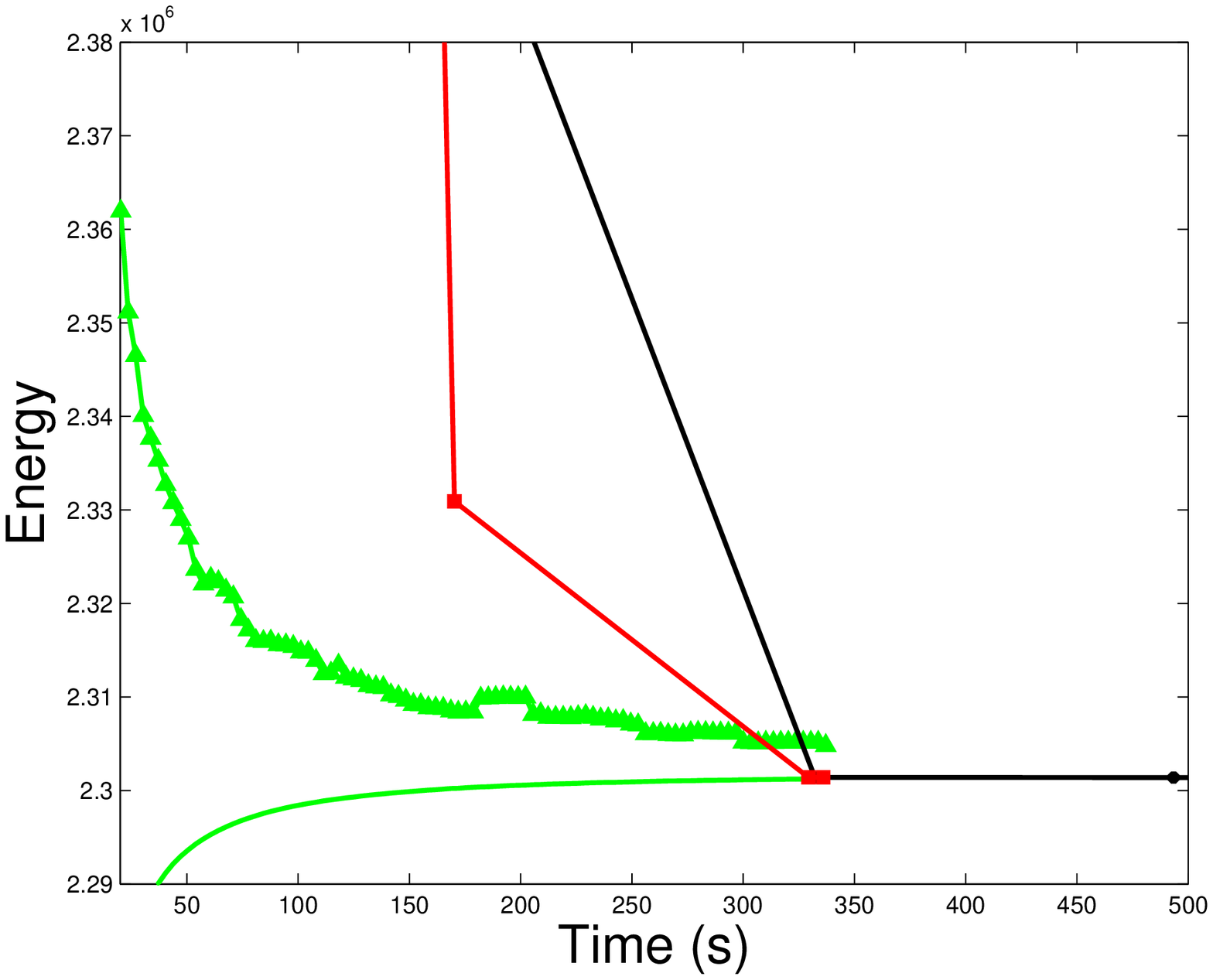}
	\caption{Cones,\\Cauchy function}
\end{subfigure}%
\begin{subfigure}{.25\textwidth}
	\includegraphics[width=0.98\linewidth]{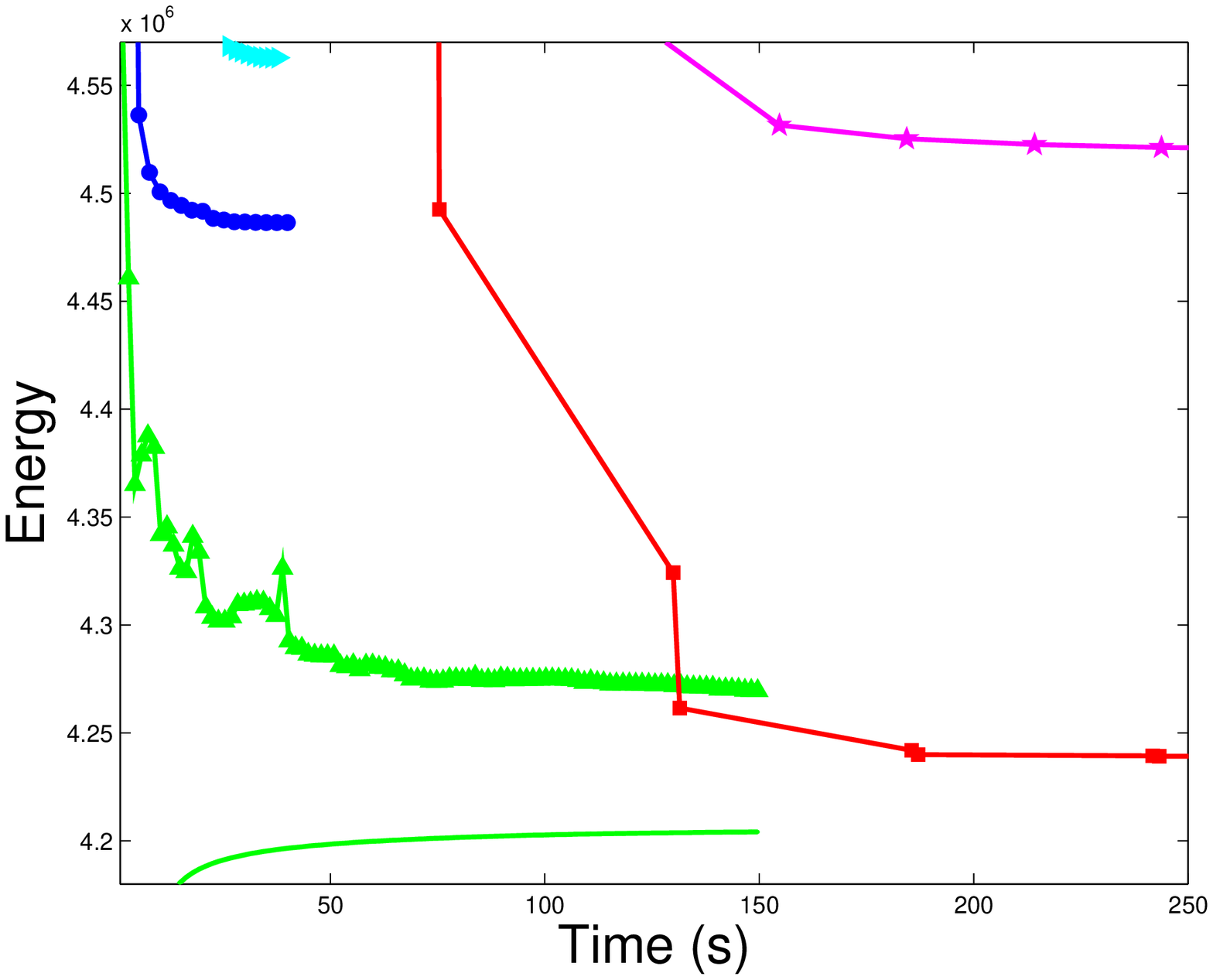}
	\caption{Penguin,\\Trunc. quadratic}
	\label{sfig:evst_pen}
\end{subfigure}
\end{center}
\vspace{-0.4cm}
 \caption{\em Energy vs time (seconds) plots for the algorithms for \textbf{(a) - (c)} some stereo problems and  \textbf{(d)} an inpainting problem. The plots are zoomed-in to show the finer details. IRGC+expansion algorithm outperformed all the other algorithms and found the lowest energy within 2--5 iterations. IRGC found a lower energy than $\alpha$-expansion for Map.}
\label{fig:evst}
\end{figure*}

\vspace{-0.3cm}
\paragraph{Inpainting:} \mbox{}\\
Image inpainting tackles the problem of filling in the missing pixel values of an image, while simultaneously denoising the observed pixel values. In our experiments, we used the Penguin and House images employed in~\cite{szeliski2008comparative}. Due to memory limitation, however, we down-sampled the labels from 256 to 128 for Penguin and from 256 to 64 for House. We used the same unary potential as in~\cite{szeliski2008comparative}, i.e., $f_p(x_p) = (I_p - x_p)^2$ if the intensity $I_p$ is observed, and $f_p(x_p) = 0$ otherwise. As pairwise potentials, we employed the truncated quadratic cost $\theta^b_{pq}(x_p, x_q) = \gamma_{pq}\,\min\left((x_p-x_q)^2, \lambda^2\right)$. For Penguin, $\gamma_{pq}=20$ and $\lambda=10$, and for House $\gamma_{pq}=5$ and $\lambda=15$. The final energies and execution times are summarized in Table \ref{tab:inp}, with the inpainted images shown in Fig. \ref{fig:pen}. Furthermore, in Fig.~\ref{sfig:evst_pen}, we show the energy as a function of time for the different algorithms. Note that for both images, our IRGC+expansion reaches the lowest energy.

\begin{table}[t]
\begin{center}
\begin{tabular}{c|cc|cc}
\multirow{2}{*}{Algorithm} & \multicolumn{2}{c}{House} & \multicolumn{2}{c}{Penguin} \\
	& E[$10^3$] & T[s] & E[$10^3$] & T[s]\\
\hline
$\alpha$-$\beta$ swap&2488.9&    47&4562.8&    38\\
$\alpha$-expansionQ&2510.0&   687&4486.4&    40\\
Multi-label swap&\textbf{2399.9}&  1442&4520.6&   392\\
TRW-S&2400.1&   130&4269.5&   150\\
IRGC&\textbf{2399.9}&   200&4696.9&  1636\\
IRGC+expansion&\textbf{2399.9}&   140&\textbf{4238.9}&   355\\
\end{tabular}
\end{center}
\vspace{-0.4cm}
\caption{\em Comparison of minimum energies (E) and execution times (T) for the truncated quadratic prior on two inpainting problems. Our IRGC+expansion algorithm outperforms all the other methods in terms of minimum energy. On House, IRGC and multi-label swap also achieve the same lowest energy, but the latter was roughly 10 times slower than IRGC+expansion. On Penguin, while IRGC was trapped in a local minimum, IRGC+expansion was able to find the lowest energy.}
\label{tab:inp}
\end{table}

\begin{figure}[t]
\begin{center}
\begin{subfigure}{.12\textwidth}
	\includegraphics[width=0.9\linewidth]{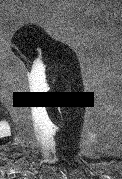}
	\caption{Input}
\end{subfigure}%
\begin{subfigure}{.12\textwidth}
	\includegraphics[width=0.9\linewidth]{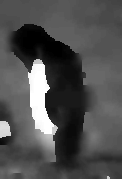}
	\caption{$\alpha$-$\beta$ swap}
\end{subfigure}%
\begin{subfigure}{.12\textwidth}
	\includegraphics[width=0.9\linewidth]{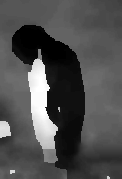}
	\caption{$\alpha$-expQ}
\end{subfigure}%
\begin{subfigure}{.12\textwidth}
	\includegraphics[width=0.9\linewidth]{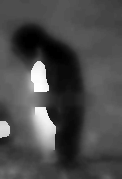}
	\caption{Multi-swap}
\end{subfigure}
\begin{subfigure}{.12\textwidth}
	\includegraphics[width=0.9\linewidth]{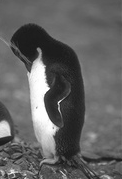}
	\caption{Ground-truth}
\end{subfigure}%
\begin{subfigure}{.12\textwidth}
	\includegraphics[width=0.9\linewidth]{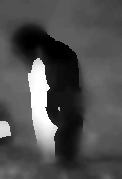}
	\caption{TRW-S}
\end{subfigure}%
\begin{subfigure}{.12\textwidth}
	\includegraphics[width=0.9\linewidth]{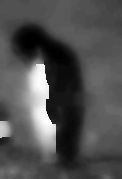}
	\caption{IRGC}
\end{subfigure}%
\begin{subfigure}{.12\textwidth}
	\includegraphics[width=0.9\linewidth]{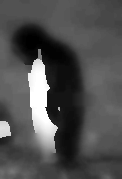}
	\caption{IRGC+exp.}
\end{subfigure}
\vspace{-0.4cm}
\end{center}
   \caption{\em Inpainted images for Penguin. Note that multi-label swap was not able to completely inpaint the missing pixels. IRGC+expansion produced smoother results than QPBOP based $\alpha$-expansion (see the bottom of the penguin) while preserving the finer details compared to TRW-S (see the neck of the penguin). See Table~\ref{tab:inp} for the energy values.}
\label{fig:pen}
\end{figure}

\vspace{-0.3cm}
\paragraph{Summary:} \mbox{}\\
To evaluate the quality of the minimum energies, we followed the strategy of~\cite{szeliski2008comparative}, which makes use of the lower bound found by TRW-S. This quality measure is computed as
\begin{equation}
Q = \frac{E - E_b}{E_b}\, 100 \%\ ,
\label{eqn:acc}
\end{equation}
where $E_b$ is the largest lower bound of TRW-S and $E$ is the minimum energy found by an algorithm. In Table~\ref{tab:acc}, we compare the resulting values of our algorithms with TRW-S, which, from the previous results was found to be the best-performing baseline. Note that our IRGC+expansion algorithm yields the best quality measure on average.

\begin{table}[t]
\begin{center}
\begin{tabular}{c|c|c|c}
Problem & TRW-S & IRGC & IRGC+exp. \\
\hline
Teddy&0.3040\%&1.6289\%&\textbf{0.2102\%}\\
Map&\textbf{0.0511\%}&0.7387\%&0.1728\%\\
Sawtooth&0.6452\%&0.9621\%&\textbf{0.2616\%}\\
Venus&0.9096\%&0.3498\%&\textbf{0.2625\%}\\
Cones&0.1551\%&\textbf{0.0065\%}&0.0074\%\\
Tsukuba&\textbf{0.0910\%}&0.4678\%&0.1679\%\\
Tsu. cor. Gau.&0.3926\%&12.1226\%&\textbf{0.2736\%}\\
Penguin&1.5556\%&11.7218\%&\textbf{0.8259\%}\\
House&0.0154\%&\textbf{0.0058\%}&\textbf{0.0058\%}\\
\hline
Average&0.4577\%&3.1116\%&\textbf{0.2431\%}\\
\end{tabular}
\end{center}
\vspace{-0.4cm}
\caption{\em Quality of the minimum energies according to Eq.~\ref{eqn:acc}. Our IRGC+expansion clearly yields the best quality energies on average.}
\label{tab:acc}
\end{table}

\section{Conclusion}
We have introduced an Iteratively Reweighted Graph Cut algorithm that can minimize multi-label MRF energies with arbitrary data terms and non-convex priors. We have shown that, while the basic algorithm sometimes gets trapped in local minima, our hybrid version consistently outperforms (or performs virtually as well as) state-of-the-art MRF energy minimization techniques. Due to our use of the multi-label graph cut at each iteration of our algorithm, memory requirement is the current bottleneck of our approach. Note, however, that this can be overcome by replacing the multi-label graph cut with the convex formulation of $\alpha$-expansion~\cite{carr2009solving}, or the approach of~\cite{pock2008convex}. In the latter case, this would further extend the applicability of IRGC to continuous label spaces. Finally, and as discussed in the paper, IRGC really is a special case of an iteratively reweighted approach to MRF, and even continuous, energy minimization. In the future, we therefore intend to study the applicability of such an approach to other types of problems.

{\small
\bibliographystyle{ieee}
\bibliography{irls}

\begin{thebibliography}{10}\itemsep=-1pt

\bibitem{aftabwacv15}
K.~Aftab and R.~Hartley.
\newblock Convergence of iteratively re-weighted least squares to robust
  m-estimators.
\newblock
  \url{http://users.cecs.anu.edu.au/~u4748671/papers/conf/irls_phuber_wacv2015.pdf}.

\bibitem{aftabgeneralized}
K.~Aftab, R.~Hartley, and J.~Trumpf.
\newblock Generalized weiszfeld algorithms for lq optimization.
\newblock
  \url{http://users.cecs.anu.edu.au/~u4748671/papers/journal/generalized_weiszfeld_algorithm_journal.pdf}.

\bibitem{birchfield1998pixel}
S.~Birchfield and C.~Tomasi.
\newblock A pixel dissimilarity measure that is insensitive to image sampling.
\newblock {\em Pattern Analysis and Machine Intelligence, IEEE Transactions
  on}, 20(4):401--406, 1998.

\bibitem{boros2002pseudo}
E.~Boros and P.~L. Hammer.
\newblock Pseudo-boolean optimization.
\newblock {\em Discrete applied mathematics}, 123(1):155--225, 2002.

\bibitem{boykov2004experimental}
Y.~Boykov and V.~Kolmogorov.
\newblock An experimental comparison of min-cut/max-flow algorithms for energy
  minimization in vision.
\newblock {\em Pattern Analysis and Machine Intelligence, IEEE Transactions
  on}, 26(9):1124--1137, 2004.

\bibitem{boykov2001fast}
Y.~Boykov, O.~Veksler, and R.~Zabih.
\newblock Fast approximate energy minimization via graph cuts.
\newblock {\em Pattern Analysis and Machine Intelligence, IEEE Transactions
  on}, 23(11):1222--1239, 2001.

\bibitem{carr2009solving}
P.~Carr and R.~Hartley.
\newblock Solving multilabel graph cut problems with multilabel swap.
\newblock In {\em Digital Image Computing: Techniques and Applications, 2009.
  DICTA'09.}, pages 532--539. IEEE, 2009.

\bibitem{felzenszwalb2006efficient}
P.~F. Felzenszwalb and D.~P. Huttenlocher.
\newblock Efficient belief propagation for early vision.
\newblock {\em International journal of computer vision}, 70(1):41--54, 2006.

\bibitem{hartley2003multiple}
R.~Hartley and A.~Zisserman.
\newblock {\em Multiple view geometry in computer vision}.
\newblock Cambridge university press, 2003.

\bibitem{ishikawa2003exact}
H.~Ishikawa.
\newblock Exact optimization for markov random fields with convex priors.
\newblock {\em Pattern Analysis and Machine Intelligence, IEEE Transactions
  on}, 25(10):1333--1336, 2003.

\bibitem{jezierska2011fast}
A.~Jezierska, H.~Talbot, O.~Veksler, and D.~Wesierski.
\newblock A fast solver for truncated-convex priors: quantized-convex split
  moves.
\newblock In {\em Energy Minimization Methods in Computer Vision and Pattern
  Recognition}, pages 45--58. Springer, 2011.

\bibitem{kappes2013comparative}
J.~H. Kappes, B.~Andres, F.~A. Hamprecht, C.~Schnorr, S.~Nowozin, D.~Batra,
  S.~Kim, B.~X. Kausler, J.~Lellmann, N.~Komodakis, et~al.
\newblock A comparative study of modern inference techniques for discrete
  energy minimization problems.
\newblock In {\em Computer Vision and Pattern Recognition (CVPR), 2013 IEEE
  Conference on}, pages 1328--1335. IEEE, 2013.

\bibitem{kolmogorov2006convergent}
V.~Kolmogorov.
\newblock Convergent tree-reweighted message passing for energy minimization.
\newblock {\em Pattern Analysis and Machine Intelligence, IEEE Transactions
  on}, 28(10):1568--1583, 2006.

\bibitem{komodakis2011mrf}
N.~Komodakis, N.~Paragios, and G.~Tziritas.
\newblock Mrf energy minimization and beyond via dual decomposition.
\newblock {\em Pattern Analysis and Machine Intelligence, IEEE Transactions
  on}, 33(3):531--552, 2011.

\bibitem{lempitsky2010fusion}
V.~Lempitsky, C.~Rother, S.~Roth, and A.~Blake.
\newblock Fusion moves for markov random field optimization.
\newblock {\em Pattern Analysis and Machine Intelligence, IEEE Transactions
  on}, 32(8):1392--1405, 2010.

\bibitem{ochs2013iterated}
P.~Ochs, A.~Dosovitskiy, T.~Brox, and T.~Pock.
\newblock An iterated l1 algorithm for non-smooth non-convex optimization in
  computer vision.
\newblock In {\em Computer Vision and Pattern Recognition (CVPR), 2013 IEEE
  Conference on}, pages 1759--1766. IEEE, 2013.

\bibitem{pock2008convex}
T.~Pock, T.~Schoenemann, G.~Graber, H.~Bischof, and D.~Cremers.
\newblock A convex formulation of continuous multi-label problems.
\newblock In {\em Computer Vision--ECCV 2008}, pages 792--805. Springer, 2008.

\bibitem{rother2007optimizing}
C.~Rother, V.~Kolmogorov, V.~Lempitsky, and M.~Szummer.
\newblock Optimizing binary mrfs via extended roof duality.
\newblock In {\em Computer Vision and Pattern Recognition, 2007. CVPR'07. IEEE
  Conference on}, pages 1--8. IEEE, 2007.

\bibitem{rother2005digital}
C.~Rother, S.~Kumar, V.~Kolmogorov, and A.~Blake.
\newblock Digital tapestry [automatic image synthesis].
\newblock In {\em Computer Vision and Pattern Recognition, 2005. CVPR 2005.
  IEEE Computer Society Conference on}, volume~1, pages 589--596. IEEE, 2005.

\bibitem{scharstein2002taxonomy}
D.~Scharstein and R.~Szeliski.
\newblock A taxonomy and evaluation of dense two-frame stereo correspondence
  algorithms.
\newblock {\em International journal of computer vision}, 47(1-3):7--42, 2002.

\bibitem{scharstein2003high}
D.~Scharstein and R.~Szeliski.
\newblock High-accuracy stereo depth maps using structured light.
\newblock In {\em Computer Vision and Pattern Recognition, 2003. Proceedings.
  2003 IEEE Computer Society Conference on}, volume~1, pages I--195. IEEE,
  2003.

\bibitem{szeliski2008comparative}
R.~Szeliski, R.~Zabih, D.~Scharstein, O.~Veksler, V.~Kolmogorov, A.~Agarwala,
  M.~Tappen, and C.~Rother.
\newblock A comparative study of energy minimization methods for markov random
  fields with smoothness-based priors.
\newblock {\em Pattern Analysis and Machine Intelligence, IEEE Transactions
  on}, 30(6):1068--1080, 2008.

\bibitem{torr2009improved}
P.~Torr and M.~Kumar.
\newblock Improved moves for truncated convex models.
\newblock In {\em Advances in neural information processing systems}, pages
  889--896, 2009.

\bibitem{veksler2012multi}
O.~Veksler.
\newblock Multi-label moves for mrfs with truncated convex priors.
\newblock {\em International journal of computer vision}, 98(1):1--14, 2012.

\bibitem{wainwright2005map}
M.~J. Wainwright, T.~S. Jaakkola, and A.~S. Willsky.
\newblock Map estimation via agreement on trees: message-passing and linear
  programming.
\newblock {\em Information Theory, IEEE Transactions on}, 51(11):3697--3717,
  2005.

\end{thebibliography}
}

\end{document}